\documentclass[sigconf]{acmart}

\usepackage{microtype}
\usepackage{graphicx}
\usepackage{subcaption}
\usepackage{booktabs} 
\usepackage{threeparttable}
\usepackage{hyperref}
\usepackage{url}
\usepackage{amsthm}

\usepackage{amssymb}
\usepackage{tikz}
\usetikzlibrary{shadows}
\usepackage{booktabs}
\usepackage{balance}
\usepackage{xcolor}
\usepackage{thmtools,thm-restate}
\usepackage[framemethod=tikz]{mdframed}
\usepackage{wrapfig}
\usepackage{threeparttable}
\usepackage{enumitem}
\usepackage{dirtytalk}
\usepackage{multirow}
\usepackage{algorithm}
\usepackage{algorithmic}
\usepackage{thmtools}
\usepackage{placeins} 

\usepackage{hyperref}

\usepackage{amsmath}
\usepackage{amssymb}
\usepackage{mathtools}
\usepackage{amsthm}

\usepackage[capitalize,noabbrev]{cleveref}

\theoremstyle{plain}
\newtheorem{theorem}{Theorem}[section]

\theoremstyle{definition}

\newtheorem{assumption}[theorem]{Assumption}
\theoremstyle{remark}

\usepackage[textsize=tiny]{todonotes}

\AtBeginDocument{%
  }

\copyrightyear{2026}
\acmYear{2026}
\setcopyright{cc}
\setcctype{by}
\acmConference[KDD '26]{Proceedings of the 32nd ACM SIGKDD Conference on Knowledge Discovery and Data Mining V.2}{August 09--13, 2026}{Jeju Island, Republic of Korea}
\acmBooktitle{Proceedings of the 32nd ACM SIGKDD Conference on Knowledge Discovery and Data Mining V.2 (KDD '26), August 09--13, 2026, Jeju Island, Republic of Korea}
\acmDOI{10.1145/3770855.3817764}
\acmISBN{979-8-4007-2259-2/2026/08}

\begin{document}

\title{Uncertainty-Calibrated Diffusion for \\ Reliable 3D Molecular Graph Generation}

\author{Fang Wan}
\orcid{0009-0007-8394-1517}
\authornote{Equal contribution.}
\email{fanwan@cs.stonybrook.edu}
\affiliation{%
  \institution{State University of New York at Stony Brook}
  \city{Stony Brook}
  \state{New York}
  \country{USA}
}
\author{Jingxiang Qu}
\orcid{0000-0001-5625-4147}
\authornotemark[1]
\email{jingxiang.qu@stonybrook.edu}
\affiliation{%
  \institution{State University of New York at Stony Brook}
  \city{Stony Brook}
  \state{New York}
  \country{USA}
}
\author{Yi Liu}
\orcid{0000-0001-7405-7972}
\email{yi.liu.4@stonybrook.edu}
\affiliation{%
  \institution{State University of New York at Stony Brook}
  \city{Stony Brook}
  \state{New York}
  \country{USA}
}

\renewcommand{\shortauthors}{Wan et al.}

\begin{abstract}
  Bayesian inference provides a principled framework for modeling epistemic uncertainty in neural networks by treating predictions as distributions rather than deterministic values. Meanwhile, diffusion-based models for 3D molecular graph generation operate on fragile geometric structures governed by strict chemical constraints, making inference highly sensitive to uncertainty miscalibration. A largely overlooked issue is that epistemic uncertainty arising from the learned denoiser interacts with the aleatoric uncertainty intentionally injected during reverse diffusion, leading to systematic variance inflation and a mismatch between the true distribution and the simulated distribution. This effect is particularly detrimental for high-precision molecular generation, where even small deviations can violate chemical validity. In this work, we provide a theoretical and empirical analysis of how epistemic uncertainty propagates through diffusion inference and degrades sampling quality. Building on this investigation, we propose \textbf{UCD} (\textbf{U}ncertainty-\textbf{C}alibrated \textbf{D}iffusion), a simple yet effective method that calibrates the reverse diffusion process to account for epistemic uncertainty. Extensive experiments on standard 3D molecular benchmarks demonstrate that UCD consistently improves sampling quality across diverse baseline methods, establishing new state-of-the-art performance for 3D molecular diffusion. The code is available at
\hyperlink{https://github.com/jiuguaiwf/UCD}{https://github.com/jiuguaiwf/UCD}.
\end{abstract}

%
%


%




\maketitle

\section{Introduction}
Generative models provide probabilistic frameworks for approximating complex data distributions and synthesizing new samples from learned generative procedures. By sampling from high-dimensional latent manifolds through a prescribed inference process, they have become powerful tools in automation, simulation, and scientific discovery. Among these approaches, diffusion models (DMs)~\cite{DDPM} have emerged as a prominent paradigm due to their effectiveness in image generation and beyond~\cite{DDIM,rfdiffusion,stablediffusion}. A DM consists of a forward process that gradually corrupts data with Markovian noise until it becomes indistinguishable from Gaussian noise, and a reverse process parameterized by a neural network (i.e., neural denoiser) that iteratively denoises these corrupted states. At inference time, the model generates new samples by simulating this learned reverse trajectory, reconstructing structured data from pure noise. 

Recent work has extended DMs to 3D molecular graph generation~\cite{EDM,GeoLDM,RADM}, where the reverse process gradually reconstructs atom types and their 3D positions. 
Molecular structures are tightly organized geometric objects, and the iterative denoising procedure must recover spatial relationships with sufficient precision for the resulting molecule to remain valid~\cite{controllable,chemicalspace,molecularspace,diffms}. In this setting, uncertainty in the neural denoiser~\cite{modeluncertainty,dropout,uncertaintysurvey,gao2024empowering} can interact with the stochastic components of the reverse diffusion process, influencing how clean geometric structure is recovered during sampling. Meanwhile, uncertainty estimation has become an important tool for characterizing model confidence in generative models~\cite{generativeuncertainty,du2023diffusion,zhang2025artificial}. Together, these perspectives motivate a natural question: \textbf{\textit{how does such uncertainty affect the reverse inference dynamics in molecular diffusion, and can it help us better characterize the stochastic behavior that arises during sampling?}}

In this work, we analyze how the uncertainty issue arises in molecular diffusion models and show that it can artificially inflate the effective noise injected during reverse inference. This phenomenon is particularly detrimental for reconstructing sensitive molecular geometries, where even small deviations can lead to chemically invalid structures. Adopting a Bayesian approximation perspective~\cite{modeluncertainty,dropout}, we introduce \underline{\textbf{UCD}} (\underline{\textbf{U}}ncertainty-\underline{\textbf{C}}alibrated \underline{\textbf{D}}iffusion), a simple yet effective correction to the sampling procedure that explicitly compensates for uncertainty-induced variance inflation, as illustrated in Fig.~\ref{fig:mainfigure}. UCD improves molecular generation quality without modifying the underlying diffusion architecture or training objective, and can be seamlessly applied to a wide range of existing diffusion backbones.

Our main contributions are threefold:
(i) we identify and formalize a previously underexplored failure mode in diffusion-based molecular generation, namely \emph{variance inflation} caused by the interaction between epistemic uncertainty in the denoiser and stochasticity in diffusion inference; (ii) we provide a theoretical analysis characterizing how this variance inflation propagates through the reverse diffusion process and degrades sampling quality, motivating a principled correction;
and (iii) we propose UCD, an inference-time uncertainty calibration method that can be readily integrated into diverse diffusion models. Extensive experiments across multiple benchmarks and diffusion backbones demonstrate that UCD consistently improves molecular generation quality without introducing additional architectural complexity.
These contributions establish UCD as a novel approach to 3D molecular diffusion that combines state-of-the-art performance with calibrated uncertainty for enhanced reliability, effectively bridging the gap between theoretical diffusion models and their practical deployment in real-world 3D molecular graph generation.

\begin{figure*}[h]
    \centering
    \includegraphics[width=\textwidth]{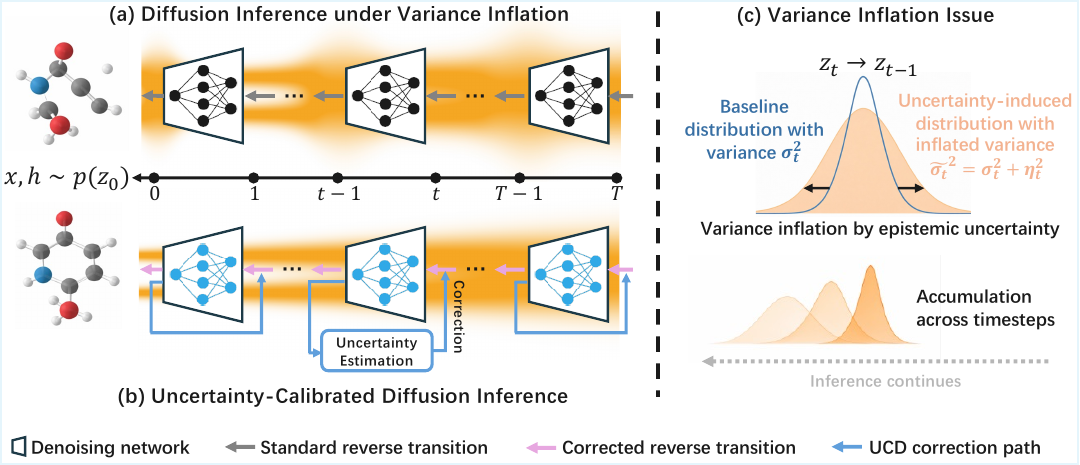}
    \caption{Illustration of variance inflation and uncertainty calibration in diffusion inference.
    \textbf{(a) Diffusion inference under variance inflation.} Standard reverse diffusion is illustrated on a sharp bimodal Gaussian mixture setting. The orange background shadow represents the evolving simulated distribution induced by the denoiser at each reverse step. Epistemic uncertainty in successive denoiser predictions causes this simulated distribution to progressively widen and become less concentrated as inference proceeds, leading to unstable final samples. \textbf{(b) Uncertainty-calibrated diffusion inference.} UCD introduces an uncertainty estimation and correction mechanism that stabilizes the reverse transitions. On the same background example, the simulated distribution remains concentrated across timesteps, resulting in consistent reverse trajectories and reliable final structures. \textbf{(c) Variance inflation issue.} At a single reverse step $\boldsymbol{z}_t \rightarrow \boldsymbol{z}_{t-1}$, epistemic uncertainty locally inflates the model transition variance, and this effect accumulates across successive reverse steps.
    }
    \label{fig:mainfigure}
    \vspace{-6pt}
\end{figure*}
\section{Preliminaries and Related Work}

\subsection{Epistemic Uncertainty}
\label{sec:model_uncertainty}


Neural networks are trained with data samples to predict the underlying function, but instead of being a solver which produces deterministic results, their outputs are modeled as Gaussian-distributed with uncertainty due to two sources: (1) inherent noise in the training data, which leads to the classical interpretation of MSE as a Gaussian likelihood~\cite{MSEtoGaussian}; (2) uncertainty in the model parameters arising from limited data and stochastic training, which leads to an approximately normal predictive distribution~\cite{dropout}. Both of these sources contribute to the predictive uncertainty of the neural network, which we refer to as \emph{epistemic uncertainty}~\cite{uncertaintysurvey}. Throughout the paper, we use the term epistemic uncertainty in this sense.

Epistemic uncertainty is typically estimated using three families of methods: deep ensembles~\cite{deepensambles,kulichenko2023uncertainty}, heteroscedastic regression~\cite{tan2023single}, and dropout-based uncertainty~\cite{dropout,gao2024empowering}. Deep ensembles require training multiple independent networks that predict the same target, and uncertainty is estimated from the variance across ensemble members. However, training several large generative backbones and selecting appropriate ensemble candidates is prohibitively expensive, making this approach impractical in diffusion settings. Heteroscedastic regression introduces an additional prediction head for the output variance, but this variance becomes part of the optimization objective and may collapse or inflate artificially. Moreover, it provides no true stochasticity and therefore cannot capture epistemic uncertainty. In diffusion models~\cite{DDPM}, predicting an extra variance term additionally interferes with the analytically defined Gaussian reverse process, often degrading generation quality. In contrast, dropout-based uncertainty~\cite{dropout} preserves the backbone architecture and performance while providing a lightweight and effective estimate of epistemic uncertainty.

Note that beyond epistemic uncertainty, \emph{aleatoric uncertainty} is also commonly discussed, which refers to irreducible randomness inherent in the data generation or sampling process~\cite{aleatoric,sabbar2025selective}. In diffusion models, aleatoric uncertainty arises from the stochastic noise injected during reverse inference. Once the diffusion inference configuration is fixed (e.g., noise schedule and sampling strategy), this source of uncertainty is fully determined and is not the focus of this work, as detailed in Sec.~\ref{sec:prelim_diff}.

\subsection{Diffusion Models and Uncertainty Estimation} \label{sec:prelim_diff}
Diffusion models~\citep{DDPM,scorebaseddiffusion} are latent-variable generative models that transform Gaussian noise into complex data samples by simulating a forward-reverse Markov chain defined over a discrete time horizon $t\in\{0,\dots,T\}$. Let $\boldsymbol{x}\sim p(\boldsymbol{x})$ denote a clean data point in $\mathbb{R}^d$, and let $\boldsymbol{z}_t$ denote its progressively noised counterpart. The forward process corrupts data by adding Gaussian noise with variance schedule $\beta_t\in(0,1)$:
\begin{equation}
\begin{aligned}
    p(\boldsymbol{z}_{1:T}\mid \boldsymbol{x})
    &= \prod_{t=1}^T p(\boldsymbol{z}_t \mid \boldsymbol{z}_{t-1}),\\
    p(\boldsymbol{z}_t \mid \boldsymbol{z}_{t-1})
    &= \mathcal{N}\!\left(\sqrt{1-\beta_t}\,\boldsymbol{z}_{t-1},\,\beta_t I_d\right).
\end{aligned}
\end{equation}
By composition, the marginal noising distribution admits a closed form,
\begin{equation}
\begin{aligned}
  &p(\boldsymbol{z}_t\mid\boldsymbol{x})=\mathcal{N}\!\left(\sqrt{\bar{\alpha}_t}\,\boldsymbol{x},\,(1-\bar{\alpha}_t)I_d\right),\\
  &\bar{\alpha}_t=\prod_{s=1}^t \alpha_s,\quad \alpha_s= 1-\beta_s,
\end{aligned}
\end{equation}
which linearly interpolates between the data distribution $p(\boldsymbol{x})$ and the Gaussian prior $p(\boldsymbol{z}_T)\approx\mathcal{N}(0,I)$~\citep{stochasticinterpolants}. The unconditional marginal at time $t$ is then
\begin{equation}
    \label{eq:true_distribution}
    p(\boldsymbol{z}_t)
    = \int p(\boldsymbol{x})\,
    \mathcal{N}\!\left(\boldsymbol{z}_t \,\middle|\, \sqrt{\bar{\alpha}_t}\,\boldsymbol{x},\,(1-\bar{\alpha}_t)I_d\right)\,d\boldsymbol{x}.
\end{equation}
After training, sampling proceeds by simulating the reverse Markov chain 
$q_\theta(\boldsymbol{z}_{0:T})
= q(\boldsymbol{z}_T)\prod_{t=1}^T q_\theta(\boldsymbol{z}_{t-1}\mid\boldsymbol{z}_t)$,
where each reverse transition is modeled as a Gaussian
\begin{equation}
    \label{eq:model_distribution}
    q_\theta(\boldsymbol{z}_{t-1}\mid\boldsymbol{z}_t)
    = \mathcal{N}\!\big(\boldsymbol{\mu}_\theta(\boldsymbol{z}_t,t),\,\sigma_t^2 I_d\big),
\end{equation}
with mean predicted by a neural network and variance $\sigma_t^2$ typically fixed according to the SDE or variance-preserving schedule.

Training is performed via the noise-prediction objective~\citep{scorebaseddiffusion}, which encourages the model to recover the forward noise:
\begin{equation}
    \label{eq:trainingloss}
\begin{aligned}
    &\mathcal{L}_{\mathrm{DM}}
    = \mathbb{E}_{\boldsymbol{x},\boldsymbol{\varepsilon},t}\,
      \big[
        \|\boldsymbol{\varepsilon}
        - \boldsymbol{\varepsilon}_\theta(\boldsymbol{z}_t,t)\|^2
      \big],
    \\
    &\boldsymbol{z}_t
    = \sqrt{\bar{\alpha}_t}\,\boldsymbol{x}
    + \sqrt{1-\bar{\alpha}_t}\,\boldsymbol{\varepsilon},
    \quad
    \boldsymbol{\varepsilon}\sim\mathcal{N}(\boldsymbol{0},I_d).
\end{aligned}
\end{equation}
Here $\boldsymbol{\varepsilon}_\theta(\boldsymbol{z}_t,t)$ approximates the score field 
$\nabla_{\boldsymbol{z}_t}\log p(\boldsymbol{z}_t)$~\citep{scoretraining,scorebaseddiffusion}. 
New samples are generated by initializing from noise $\boldsymbol{z}_T\sim\mathcal{N}(0,I_d)$ and iteratively applying the reverse update:
\begin{equation}
\label{eq:reverse_step_uncertainty}
    \boldsymbol{z}_{t-1}
    = \frac{1}{\sqrt{1-\beta_t}}
      \left(\boldsymbol{z}_t
      - \frac{\beta_t}{\sqrt{1-\bar{\alpha}_t}}\,\boldsymbol{\varepsilon}_\theta(\boldsymbol{z}_t,t)\right)
      + \sigma_t \boldsymbol{\varepsilon}.
\end{equation}
Diffusion-based generative models intentionally inject stochastic noise $\boldsymbol{\varepsilon}$ at each reverse step through the simulation of stochastic differential equations (SDEs)~\cite{DDPM,scorebaseddiffusion}. This stochasticity causes the aleatoric uncertainty, which is intrinsic and fully determined once the diffusion inference configuration is fixed.

Epistemic uncertainty, in contrast, originates from the learned denoiser $\boldsymbol{\varepsilon}_\theta$ and reflects imperfect modeling of the data and the training process. \textbf{These two forms of uncertainty, aleatoric uncertainty induced by the SDE simulation and epistemic uncertainty stemming from the learned model, coexist during the reverse diffusion process.} However, in standard diffusion inference, epistemic uncertainty is typically ignored, as inference relies on a single deterministic denoiser prediction optimized to minimize the training objective (e.g., Eq.~\eqref{eq:trainingloss}). \textbf{Their interaction was accordingly overlooked in prior works, which distort the reverse inference dynamics with accumulated errors across timesteps, ultimately leading to degraded sample quality.}

\subsection{Relations with Prior Work}
\label{similarworks}

Recent years have seen substantial progress in 3D molecular generation, with diffusion-based and flow-based models emerging as a dominant paradigm for generating 3D molecular graphs. Representative works model a 3D molecule with $N$ atoms as a combination of discrete atomic features 
\(\boldsymbol{h}=(\boldsymbol{h}_1,\dots,\boldsymbol{h}_N)\in \mathbb{R}^{N\times d}\) 
and continuous three-dimensional coordinates 
\(\boldsymbol{x}=(\boldsymbol{x}_1,\dots,\boldsymbol{x}_N)\in \mathbb{R}^{N\times 3}\), 
defining the molecule directly on its original data manifold~\citep{liu2022spherical,wang2022comenet}. 
Such approaches (e.g., EDM, GeoLDM, and RADM) achieve strong performance 
in terms of chemical validity and stability~\citep{EDM, GeoLDM, RADM}.
Building upon these generative frameworks, subsequent studies explore improved probability paths to enhance molecular generation quality and sampling efficiency~\citep{GeoBFN, SLDM,feng2025unigem,gao2026scaling}. Another line of research focuses on developing more effective sampling strategies to improve sample quality or reduce inference cost~\citep{qu2026gaga, Revisiting_Sampling}. While these methods improve sampling efficiency or geometric fidelity, they generally treat the neural denoiser as deterministic at inference time. \textbf{Consequently, an imperfectly learned denoiser with non-negligible predictive uncertainty can induce deviations from the ideal probability path. Such deviations may accumulate across sampling steps and lead to a progressive drift away from the data manifold.}

Motivated by the impact of uncertainty on model reliability, uncertainty estimation has been increasingly explored in diffusion models as a tool to improve robustness and efficiency. One line of work leverages uncertainty estimation for \emph{active learning} and \emph{data selection}, where uncertainty serves as a data-driven criterion for guiding data acquisition, annotation, or training-time decisions~\cite{huuncertain,du2023diffusion,barba2025diffusion}. In active domain adaptation, uncertainty extracted from a diffusion model is used to guide the selection of unlabeled samples for annotation or refinement, enabling more effective adaptation under limited supervision~\cite{du2023diffusion}. In safety-critical applications, such as diffusion-based PDE control, uncertainty estimates act as indicators of sample reliability, allowing only confident generations to be retained during fine-tuning~\cite{huuncertain}. Similarly, when diffusion models rely on costly or scarce data sources, such as computed tomography measurements, uncertainty-aware acquisition strategies adaptively select the most informative data points, reducing acquisition cost while maintaining reconstruction quality~\cite{barba2025diffusion}. These approaches primarily use uncertainty to guide \emph{training-time} decisions or data acquisition, whereas UCD operates as an \emph{inference-time} correction and is therefore orthogonal in scope.

Another related direction employs uncertainty during inference to filter undesirable trajectories or steer diffusion toward more confident regions by refining denoiser outputs~\cite{jiang2025unpose,kou2023bayesdiff,de2025diffusion,generativeuncertainty}. For example, in object pose estimation, uncertainty derived from a diffusion prior is used to guide alignment and refinement for downstream tasks~\cite{jiang2025unpose}. More generally, uncertainty is treated as a semantic or decision-level signal for selecting reliable samples~\cite{generativeuncertainty} or reweighting predictions~\cite{kou2023bayesdiff,de2025diffusion}. These methods can be viewed as forms of uncertainty-guided inference that keep the reverse diffusion dynamics fixed, using uncertainty only for sample-level selection or post hoc refinement. Consequently, they do not address how epistemic uncertainty alters the underlying reverse transition kernel, nor do they provide a distribution-level characterization of uncertainty accumulation across timesteps. Moreover, many such approaches rely on last-layer Laplace approximations~\cite{LLLA}, which capture only local, readout-level uncertainty under strong linearity assumptions and may be limited for highly nonlinear diffusion dynamics.

In contrast, UCD explicitly identifies epistemic uncertainty as a structural source of variance inflation in the reverse diffusion process itself. Rather than using uncertainty for data selection, sample filtering, or post hoc refinement, \textbf{UCD leverages uncertainty as a mechanism to correct and calibrate the reverse sampling process itself}. This distinction matters for molecular diffusion because small uncertainty-induced perturbations can accumulate along the sampling trajectory and produce invalid geometries. Overall, UCD contributes to the broader study of uncertainty in diffusion models by identifying variance inflation induced by epistemic uncertainty, characterizing its accumulation and impact on sampling quality, and introducing an effective calibration strategy that consistently improves performance across diverse 3D molecular diffusion backbones.

\section{Method}
In this section, we first present empirical evidence revealing the overlooked uncertainty issue in existing diffusion-based molecular generation methods (Sec.~\ref{uncertainty_issue}). We then provide a theoretical analysis of how epistemic uncertainty degrades inference performance during diffusion sampling (Sec.~\ref{inflation}). Finally, we introduce an uncertainty correction strategy UCD that addresses this issue and enables stable and reliable molecular generation (Sec.~\ref{sec:uncertainty_correction}).

\subsection{Uncertainty Issue}
\label{uncertainty_issue}

During diffusion sampling, stochastic noise is injected at each reverse step according to the prescribed diffusion SDE, yielding a fixed level of aleatoric uncertainty~\cite{scorebaseddiffusion,DDIM}. Beyond this analytically defined stochasticity, the reverse transition additionally depends on the denoiser prediction $\boldsymbol{\varepsilon}_\theta(\boldsymbol{z}_t,t)$, whose output is subject to epistemic uncertainty. In practice, standard diffusion inference treats this prediction as deterministic, ignoring uncertainty in the learned reverse drift. As introduced in Sec.~\ref{sec:model_uncertainty}, we model this epistemic uncertainty using a Gaussian approximation and estimate it via a dropout-based Bayesian approach.

In Sec.~\ref{inflation}, we show that uncertainty in the denoiser output does not merely introduce local prediction error, but instead induces an effective inflation of the reverse transition variance. We formalize this phenomenon as \textbf{variance inflation}. As illustrated in Fig.~\ref{fig:mainfigure}(c), epistemic uncertainty enlarges the transition variance at each reverse step, and this effect accumulates along the sampling trajectory. Consequently, the simulated distribution becomes progressively more dispersed (Fig.~\ref{fig:mainfigure}(a)), leading to unstable reverse trajectories and degraded sample quality.

\begin{figure}[h]
    \centering    
\includegraphics[width=\linewidth]{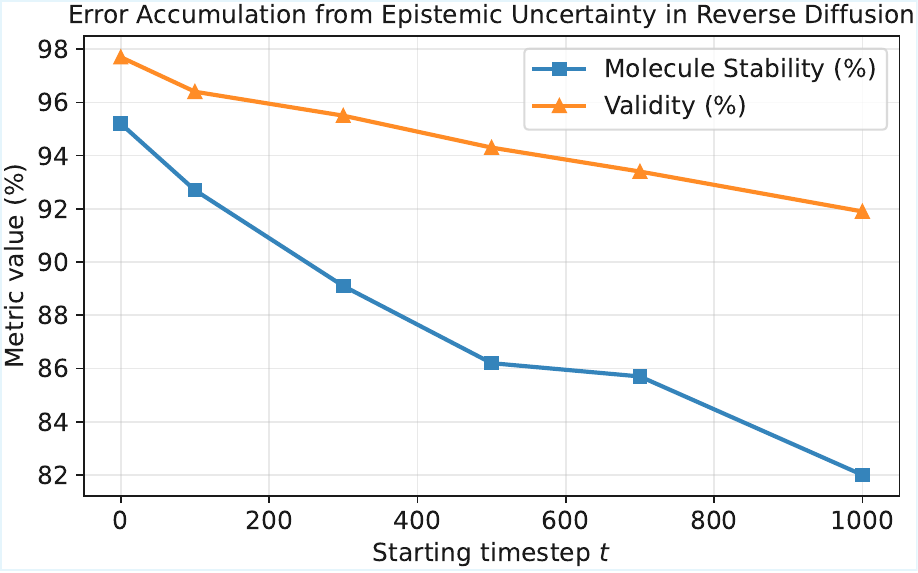}
    \caption{Effect of starting timestep $t$ on molecular generation quality. Samples are initialized from the true forward diffusion distribution at timestep $t$ and then denoised using the learned reverse process. Because the initial state is drawn from the true distribution, performance degradation directly reflects error accumulation induced by epistemic uncertainty in the reverse dynamics. Higher is better for both metrics. Experimental setup follows EDM~\citep{EDM} on QM9~\citep{QM9}, as detailed in Sec.~\ref{experimentsetting}. }
    \label{fig:start_timestep}
\end{figure}
Beyond the effect of distributional dispersion, we further study how the interaction between aleatoric and epistemic uncertainty influences the entire generation trajectory by quantifying the accumulated impact of variance inflation. In Sec.~\ref{sec: Path-space KL}, we analyze this interaction at the trajectory level and show that epistemic uncertainty, when coupled with aleatoric noise in the reverse process, induces a systematic and accumulating distortion of the sampling dynamics. This accumulation substantially enlarges the discrepancy between the simulated and true distributions. As empirically verified in Fig.~\ref{fig:start_timestep}, \textbf{generation quality degrades monotonically as the starting timestep $t$ increases, indicating that the interaction between aleatoric and epistemic uncertainty is consistently present at each reverse transition and accumulates along the probability path rather than arising from isolated or transient errors.} Building on this observation, the following sections formally characterize the accumulation behavior of this interaction along the reverse diffusion process by deriving trajectory level discrepancy measures and end-point deviation bounds, and subsequently propose a principled correction strategy based on their interaction rules during diffusion sampling.

\subsection{Variance Inflation}

\subsubsection{Theoretical Characterization}
\label{inflation}
As discussed above, aleatoric uncertainty from the diffusion SDE and epistemic uncertainty from the learned denoiser coexist during reverse sampling and jointly influence the inference dynamics. In this section, we show that their interaction admits a precise statistical characterization: epistemic uncertainty at inference time effectively enlarges the injected noise in diffusion sampling. We analyze how this additional variance propagates through the reverse chain and induces a non-trivial mismatch at the end-point $t=0$. Throughout this analysis, the trained network parameters $\theta$ and the diffusion schedule $\{\alpha_t,\beta_t\}_{t=1}^T$ are fixed. For clarity, we use $p$ to denote the true (ideal) distribution and $\tilde p$ to denote the simulated distribution induced by sampling with the learned model $\theta$ under epistemic uncertainty.

Consider a single step reverse update as defined in Eq.~\eqref{eq:reverse_step_uncertainty}, we define the reverse mean map
\begin{equation}
f_t(\boldsymbol{z}_t)
~:=~
\frac{1}{\sqrt{\alpha_t}}
\Bigl(
\boldsymbol{z}_t - \frac{\beta_t}{\sqrt{1-\bar\alpha_t}}\,\boldsymbol{\varepsilon}_\theta(\boldsymbol{z}_t,t)
\Bigr),
\label{eq:def_ft}
\end{equation}
so that Eq.~\eqref{eq:reverse_step_uncertainty} can be written as $\boldsymbol{z}_{t-1}=f_t(\boldsymbol{z}_t)+\sigma_t \boldsymbol{\varepsilon}_t$. For convenience, we define the schedule-dependent coefficient
\begin{equation}
\kappa_t ~:=~ \frac{1}{\sqrt{\alpha_t}}\frac{\beta_t}{\sqrt{1-\bar\alpha_t}}.
\label{eq:def_kappa}
\end{equation}
Bayesian inference provides a principled framework for reasoning about epistemic uncertainty~\cite{modeluncertainty,dropout}. A standard practice is to model epistemic uncertainty by augmenting the deterministic output of the learned denoiser $\boldsymbol{\varepsilon}_\theta$ with additive Gaussian noise:
\begin{equation}
\tilde{\boldsymbol{\varepsilon}}_\theta(\boldsymbol{z}_t,t)=\boldsymbol{\varepsilon}_\theta(\boldsymbol{z}_t,t)+\boldsymbol{\delta}_t,
\qquad
\boldsymbol{\delta}_t\sim\mathcal N(\boldsymbol{0},\tau_t^2 I_d),
\label{eq:uncertainty_model}
\end{equation}
where $\tau_t^2$ may be timestep-dependent, capturing heteroscedastic epistemic uncertainty.
Substituting Eq.~\eqref{eq:uncertainty_model} into the reverse update in Eq.~\eqref{eq:reverse_step_uncertainty} and using the definitions of $f_t$ and $\kappa_t$, we obtain
\begin{equation}
\boldsymbol{z}_{t-1}=f_t(\boldsymbol{z}_t) - \kappa_t \boldsymbol{\delta}_t + \sigma_t \boldsymbol{\varepsilon} .
\label{eq:reverse_with_uncertainty}
\end{equation}
Let $\eta_t^2 := \kappa_t^2\tau_t^2$, then we have $\kappa_t \boldsymbol{\delta}_t\sim\mathcal N(\boldsymbol{0},\eta_t^2 I_d)$
and the following characterization of the effective reverse transition.

\begin{restatable}[Effective reverse transition under epistemic uncertainty]{proposition}{effectivekernel}
\label{prop:effective_kernel}
Under Eq.~\eqref{eq:uncertainty_model}, the conditional distribution of $\boldsymbol{z}_{t-1}$ given $\boldsymbol{z}_t$ is
\begin{equation}
\begin{aligned}
\tilde K_t(
\boldsymbol{z}_{t-1}\mid \boldsymbol{z}_t)
&= \mathcal N\!\bigl(
\boldsymbol{z}_{t-1};\, f_t(\boldsymbol{z}_t),\, \tilde\sigma_t^2 I_d
\bigr),
\end{aligned}
\label{eq:unc_kernel}
\end{equation}
where $\tilde\sigma_t^2 := \sigma_t^2 + \eta_t^2$.
And the true (no-epistemic-uncertainty) reverse kernel is
\begin{equation}
K_t(\boldsymbol{z}_{t-1}\mid \boldsymbol{z}_t)
~=~
\mathcal N\!\bigl(\boldsymbol{z}_{t-1};\, f_t(\boldsymbol{z}_t),\,\sigma_t^2 I_d\bigr).
\label{eq:base_kernel}
\end{equation}
\end{restatable}
The proof is provided in Appendix~\ref{proof3_1}. Proposition~\ref{prop:effective_kernel} shows that inference-time epistemic uncertainty induces a precise and unavoidable \emph{inflation of the reverse transition variance} from $\sigma_t^2$ to $\tilde\sigma_t^2$ at every timestep. Proposition~\ref{prop:effective_kernel} admits a simple interpretation.
At inference time, epistemic uncertainty in the denoiser behaves as an additional source of noise that is injected into each reverse diffusion step.
Although the diffusion SDE prescribes a nominal noise variance $\sigma_t^2$, uncertainty in the model prediction effectively perturbs the reverse update, resulting in an enlarged transition variance $\tilde\sigma_t^2 = \sigma_t^2 + \eta_t^2$.

From this perspective, \textbf{epistemic uncertainty does not merely introduce randomness in the denoiser output, but systematically increases the stochasticity of the reverse kernel itself.}
As illustrated in Fig.~\ref{fig:mainfigure}(a), even when the reverse mean $f_t(\boldsymbol{z}_t)$ is accurate on average, epistemic uncertainty in successive predictions causes the simulated distribution to spread more than intended.
While this inflation may appear modest at a single step, its repeated application across timesteps fundamentally alters the behavior of the reverse diffusion process.
A natural question then arises: \textit{how does such local variance inflation translate into a mismatch between the final sampling
distribution $\tilde p_0$ and the true data distribution $p_0$?}


To obtain results that are both rigorous and assumption-light, we proceed in two complementary steps: (i) we first quantify discrepancy at the level of \emph{entire reverse trajectories} using path-space KL, which admits an exact closed form and captures the cumulative effect of variance inflation across timesteps; and (ii) we then analyze how epistemic uncertainty propagates to the end-point via a \emph{trajectory-local} perturbation recursion, yielding a concrete lower bound on the end-point deviation that directly characterizes how this discrepancy manifests in the final samples under mild non-collapse conditions.
\vspace{-1mm}
\subsubsection{Trajectory-Level Divergence via Path-Space KL}
\label{sec: Path-space KL}

Let $P^{\text{base}}(\boldsymbol{z}_{0:T})$ and $P^{\text{unc}}(\boldsymbol{z}_{0:T})$ denote the \emph{trajectory} distributions generated by the Markov chains using kernels $\{K_t\}$ and $\{\tilde K_t\}$, respectively, starting from the same initialization $p_T$. Let $p_t$ and $\tilde p_t$ denote the corresponding marginals. We define the relative variance inflation ratio
\begin{equation}
r_t ~:=~ \frac{\eta_t^2}{\sigma_t^2}
~=~
\frac{\kappa_t^2\tau_t^2}{\sigma_t^2},
\qquad r_t\ge 0.
\label{eq:def_rt}
\end{equation}
Since $K_t(\cdot\mid \boldsymbol{z}_t)$ and $\tilde K_t(\cdot\mid \boldsymbol{z}_t)$ are Gaussians
with the same mean and isotropic covariances, for any fixed $\boldsymbol{z}_t$,
\begin{equation}
\begin{aligned}
&\mathrm{KL}\!\left(
\tilde K_t(\cdot\mid \boldsymbol{z}_t)\,\|\,K_t(\cdot\mid \boldsymbol{z}_t)
\right)
 \\&=
\mathrm{KL}\!\left(
\mathcal N\!\left(\boldsymbol{0},(1+r_t)\sigma_t^2 I_d\right)
\,\|\,\mathcal N\!\left(\boldsymbol{0},\sigma_t^2 I_d\right)
\right) \\
&= \frac{d}{2}\Bigl(r_t-\log(1+r_t)\Bigr).
\end{aligned}
\label{eq:per_step_kl}
\end{equation}
Notably, Eq.~\eqref{eq:per_step_kl} depends \emph{only} on the inflation ratio $r_t$ and not on the nonlinear map $f_t$.
\begin{restatable}[Exact KL divergence on trajectory space]{proposition}{pathkl}
\label{thm:path_kl}
Assuming both chains share the same initialization $p_T$, the KL divergence between trajectory measures satisfies
\begin{equation}
\mathrm{KL}\!\left(P^{\text{unc}}(\boldsymbol{z}_{0:T})\,\|\,P^{\text{base}}(\boldsymbol{z}_{0:T})\right)
=
\sum_{t=1}^T \frac{d}{2}\Bigl(r_t-\log(1+r_t)\Bigr).
\label{eq:path_kl_sum}
\end{equation}
\end{restatable}
We give the derivation in Appendix~\ref{proof3_2}. Proposition~\ref{thm:path_kl} highlights the central role of the variance inflation ratio $r_t$ in determining overall inference behavior. In particular, the path-space KL divergence grows additively across timesteps and is given by
Eq.~\eqref{eq:path_kl_sum}.
Whenever the relative variance inflation $r_t$ is non-negligible over a nontrivial portion of the
reverse trajectory, the quantity
$\sum_{t=1}^T \frac{d}{2}\Bigl(r_t-\log(1+r_t)\Bigr)$
becomes large, certifying a significant divergence between the induced and ideal sampling procedures.

Proposition~\ref{thm:path_kl} makes explicit that the effect of variance inflation is cumulative rather than local.
Although the additional variance $\eta_t^2$ introduced at each step may be small, its contribution to the path-space KL divergence accumulates additively over the $T$ reverse diffusion steps.
\textbf{As a result, even mild epistemic uncertainty, when present across a nontrivial portion of the trajectory, certifies a substantial divergence between the induced sampling process and the ideal reverse diffusion.}

This accumulation perspective provides a direct explanation for the empirical behavior observed in Fig.~\ref{fig:start_timestep}.
As the number of reverse steps increases, molecular quality degrades monotonically, even when initialization is drawn from the true forward distribution.
The theory thus confirms that small uncertainty-induced deviations introduced early in the trajectory propagate and amplify over time, ultimately leading to significant loss of molecular stability and validity at the final timestep.

\vspace{-1mm}
\subsubsection{Local Propagation to the End-point}
While the path-space KL divergence certifies a substantial discrepancy between inference procedures, it does not directly explain how injected variance propagates to the final sample. We therefore complement the KL analysis with a trajectory-local perturbation propagation argument, which explicitly tracks how uncertainty-induced noise is transformed through the nonlinear reverse mapping
and accumulates into a non-negligible end-point deviation.

Consider the ideal sampler and its uncertainty-inflated counterpart under a synchronous coupling:
\begin{equation}
\begin{aligned}
\boldsymbol{Z}_{t-1} &= f_t(\boldsymbol{Z}_t) + \sigma_t \boldsymbol{\varepsilon}, \qquad
\tilde{\boldsymbol{Z}}_{t-1} = f_t(\tilde {\boldsymbol{Z}}_t) + \sigma_t \boldsymbol{\varepsilon} + \eta_t \boldsymbol{u}_t,
\end{aligned}
\label{eq:coupled_chains}
\end{equation}
where $\eta_t=\kappa_t\tau_t$ and all $\boldsymbol{u}_t \sim \mathcal N(\boldsymbol{0},I_d)$ are independent.
Let $\Delta_t:=\tilde{\boldsymbol{Z}}_t-\boldsymbol{Z}_t$ (with $\Delta_T=0$ if initialized identically). Subtracting Eq.~\eqref{eq:coupled_chains} gives the exact recursion
\begin{equation}
\Delta_{t-1}=\bigl(f_t(\boldsymbol{Z}_t+\Delta_t)-f_t(\boldsymbol{Z}_t)\bigr)+\eta_t \boldsymbol{u}_t.
\label{eq:delta_recursion}
\end{equation}
Eq.~\eqref{eq:delta_recursion} makes explicit the two sources of deviation: transport of existing perturbations through the reverse mean map, and fresh isotropic noise injected by epistemic uncertainty at timestep $t$. To quantify how deviations propagate, we measure the sensitivity of $f_t$ along the trajectory. 
By the mean value theorem, for each realization there exists a matrix
\begin{equation}
J_t := \int_0^1 \nabla f_t\!\left(\boldsymbol{Z}_t+s\Delta_t\right)\,ds,
\label{eq:def_Jt}
\end{equation}
such that $f_t(\boldsymbol{Z}_t+\Delta_t)-f_t(\boldsymbol{Z}_t)=J_t\,\Delta_t$ and hence $\Delta_{t-1}=J_t\Delta_t+\eta_t \boldsymbol{u}_t.$
The matrix $J_t$ is a trajectory-dependent linearization of the reverse mean map along the segment connecting $\boldsymbol{Z}_t$ and $\tilde{\boldsymbol{Z}}_t$. To rule out degenerate cases in which the reverse dynamics annihilate all perturbations, we impose a mild non-collapse condition on the sensitivity of $f_t$ along the trajectory.
\begin{assumption}[Bounded sensitivity of the reverse mean map]
\label{ass:sensitivity}
There exist constants $0\le m_t\le L_t$ such that
\begin{equation}
m_t \le \lambda_{\min}(J_t)
\qquad\text{and}\qquad
\lambda_{\max}(J_t)\le L_t,
\label{eq:sensitivity_bounds}
\end{equation}
where $\lambda_{\min}(\cdot)$ and $\lambda_{\max}(\cdot)$ denote the minimum and maximum singular values.
\end{assumption}

Intuitively, this assumption imposes positive lower and upper bounds on the singular values of $J_t$, which quantify how perturbations are transformed by the reverse mean map at each timestep. The lower bound prevents the reverse dynamics from completely collapsing perturbations, while the upper bound rules out uncontrolled amplification of small deviations. Together, these bounds ensure that uncertainty is transported in a stable and well-conditioned manner along the reverse trajectory. It is consistent with empirical observations in diffusion models, where the reverse dynamics are locally smooth and do not exhibit degenerate collapse or explosion on the data manifold although the learned score function is imperfect.

\begin{restatable}[End-point deviation induced by variance inflation]{proposition}{endpointdeviation}
\label{prop:endpoint_deviation}
Under Eq.~\eqref{eq:def_Jt} and Assumption~\ref{ass:sensitivity},
the end-point deviation under the synchronous coupling satisfies
\begin{equation}
d\sum_{t=1}^T \eta_t^2\prod_{s=1}^{t-1} m_s^2
~\le~
\mathbb E\|\Delta_0\|^2
~\le~
d\sum_{t=1}^T \eta_t^2\prod_{s=1}^{t-1} L_s^2,
\label{eq:endpoint_bounds}
\end{equation}
where the empty product is defined as $1$.
\end{restatable}
The derivation is provided in Appendix~\ref{proof3_4}.
Proposition~\ref{prop:endpoint_deviation} characterizes how epistemic uncertainty injected at each reverse step accumulates into an end-point deviation through the reverse dynamics. The lower bound establishes that, in the absence of degeneracy in the reverse mean map, uncertainty-induced perturbations necessarily lead to a non-negligible deviation at the final sample.
Importantly, the upper bound reveals that this deviation does not grow arbitrarily:
the total error is governed by a multiplicative accumulation of per-step uncertainty,
scaled by the local sensitivity of the reverse mean maps. This shows that the impact of epistemic uncertainty is systematic and controlled, rather than unstable or explosive.
\textbf{Together, these bounds imply that while inference-time epistemic uncertainty inevitably degrades sample quality, its effect can be effectively regulated through bounded, step-wise corrections to the reverse transition.} This observation directly motivates the uncertainty-aware noise calibration strategy
introduced in Sec.~\ref{sec:uncertainty_correction}.

Together with the path-space KL analysis, this result provides complementary theoretical evidence that inference-time epistemic uncertainty induces both a substantial algorithmic shift and a tangible, yet controllable, degradation of final sample quality.

\subsection{Uncertainty-Calibrated Diffusion}
\label{sec:uncertainty_correction}
As established in Proposition~\ref{prop:effective_kernel}, the core issue arises from \emph{epistemic uncertainty inflating the effective variance of the reverse transition}.
As a result, the stochasticity injected during inference is no longer purely aleatoric, but instead becomes a mixture of aleatoric noise and accumulated epistemic uncertainty, leading to unstable reverse dynamics and degraded sample quality.

To address this issue, we propose \textbf{UCD}, an \textbf{U}ncertainty-\textbf{C}alibrated \textbf{D}iffusion reverse sampling strategy grounded in the analysis of Sec.~\ref{inflation}. A conceptual illustration of the effect of UCD is shown in Fig.~\ref{fig:mainfigure}(b).
Proposition~\ref{prop:effective_kernel} characterizes how epistemic uncertainty induces systematic variance inflation at each reverse step, while
Propositions~\ref{thm:path_kl}  and~\ref{prop:endpoint_deviation} further show that this inflated variance accumulates through the reverse dynamics in a controlled manner.
Motivated by these results, we correct the reverse transition at each timestep by adjusting the aleatoric uncertainty variance according to the estimated epistemic uncertainty.
Concretely, we modify the scale of the stochastic noise injected by the sampler such that the resulting effective variance compensates for epistemic inflation, while leaving the reverse mean map unchanged.
This step-wise calibration restores a balanced uncertainty level during inference and stabilizes the reverse sampling trajectory.

Our analysis in Sec.~\ref{inflation} is conducted under the standard DDPM-based
reverse diffusion framework~\cite{DDPM,scorebaseddiffusion}, which is also the default inference setting for most 3D molecular diffusion models~\cite{EDM,GeoLDM,RADM}.
Although our primary focus is molecular generation, where strict chemical constraints
demand especially precise reverse inference, the proposed correction naturally extends to more general sampling schemes.
In particular, we consider a generalized inference formulation~\cite{DDIM}, which interpolates between stochastic SDE-based sampling and deterministic ODE-based sampling
via a continuous interpolation parameter.

In practice, we estimate epistemic uncertainty using dropout-based Bayesian inference~\cite{dropout}, which has been shown to provide a lightweight yet effective approximation.
At each reverse timestep, the estimated epistemic variance is used to adjust the
aleatoric uncertainty scale, yielding a corrected reverse update that stabilizes the sampling
trajectory. This correction is especially beneficial for 3D molecular generation, where small accumulated errors can easily violate geometric or chemical constraints. The complete uncertainty-calibrated generalized sampler is summarized in Algorithm~\ref{alg:UCGS}.

\begin{algorithm}[t]
\begingroup
\small
\caption{Uncertainty-Calibrated Generalized Sampler}
\label{alg:UCGS}
\begin{algorithmic}
   \STATE \textbf{Inputs:} dropout denoiser $\boldsymbol{\varepsilon}_\theta$, noise schedule $\{\beta_t,\alpha_t\}_{t=1}^T$, interpolation coefficients $\{\rho_t\}\in[0,1]^T$, dropout passes $M$, aleatoric noise adjustment function $\Phi_t(\cdot,\cdot)$, uncertainty estimator $U(\cdot,M)$
   \STATE \textbf{Output:} generated sample $\boldsymbol{x},\boldsymbol{h} \sim p(\boldsymbol{x},\boldsymbol{h}\mid\boldsymbol{z}_0)$
   \STATE Sample initial noise $\boldsymbol{z}_T \sim \mathcal{N}(\boldsymbol{0},I_d)$

   \FOR{$t = T, T-1, \dots, 1$}

      \STATE \textbf{(Dropout uncertainty estimation)}
      \STATE $\tilde{\boldsymbol{\varepsilon}}_\theta(\boldsymbol{z}_t,t),\;\tau_t\gets U(\boldsymbol{\varepsilon}_\theta(\boldsymbol{z}_t,t), M)$
      \STATE $\tilde{\beta}_t \gets \beta_t\frac{1-\bar{\alpha}_{t-1}}{1-\bar{\alpha}_t}$
      \STATE $\kappa_t \gets \sqrt{1-\bar{\alpha}_{t-1}-{\rho_t}^2\tilde{\beta}_t} \;\tau_t $
      \STATE $\sigma_t \gets \rho_t \sqrt{\tilde{\beta}_t}$
      \STATE \textbf{(Adjustment for aleatoric noise)}
      \STATE $\sigma_t^{\mathrm{adj}} \gets \Phi_t(\kappa_t,\sigma_t)$
      
      \STATE $\hat{\boldsymbol{z}}_0 \gets \dfrac{\boldsymbol{z}_t-\sqrt{1-\bar{\alpha}_t}\,\tilde{\boldsymbol{\varepsilon}}_\theta(\boldsymbol{z}_t,t)}{\sqrt{\bar{\alpha}_t}}$

      \STATE Sample $\boldsymbol{\varepsilon} = (\boldsymbol{\varepsilon}_{x}, \boldsymbol{\varepsilon}_{h}) \sim \mathcal{N}(\boldsymbol{0}, I_d)$
      \STATE Subtract center of mass from $\boldsymbol{\varepsilon}_{x}$
     \STATE \textbf{(Corrected reverse update)}
     \STATE $\boldsymbol{z}_{t-1} \gets \sqrt{\bar{\alpha}_{t-1}}\,\hat{\boldsymbol{z}}_0+\sqrt{1-\bar{\alpha}_{t-1}-{\rho_t}^2\tilde{\beta}_t} \; \tilde{\boldsymbol{\varepsilon}}_\theta(\boldsymbol{z}_t,t)+\sigma_t^{\mathrm{adj}}\boldsymbol{\varepsilon}$
   \ENDFOR
\end{algorithmic}
\endgroup
\end{algorithm}

\section{Experiments}
\label{experimentsetting}
\begin{table*}[t]
\centering
\caption{Results for atom stability, molecule stability, validity, and validity$\times$uniqueness. Baseline results are referred from the original papers.
$\uparrow$ represents higher values indicating better performance. 
All models combined with UCD surpass their original counterparts, and the improved results are shown in \textbf{bold}. 
Globally best results are \underline{underlined}. 
}
\label{tab:mainresults}
\resizebox{1\textwidth}{!}{
\begin{threeparttable}
\resizebox{\linewidth}{!}{
\begin{tabular}{l | c c c c | c c}
    \toprule[1.0pt]
    & \multicolumn{4}{c|}{\shortstack[c]{\textbf{QM9}}} & \multicolumn{2}{c}{\shortstack[c]{\textbf{GEOM-Drugs}}} \\
    \# Metrics & Atom Sta (\%) $\uparrow$ & Mol Sta (\%) $\uparrow$ & Valid (\%) $\uparrow$ & Valid$\times$Unique (\%) $\uparrow$ & Atom Sta (\%) $\uparrow$ & Valid (\%) $\uparrow$ \\
    \midrule[0.8pt]
    Data &  99.0 & 95.2 & 97.7 & 97.7 & 86.5 & 99.9 \\
    \midrule
    {ENF} &  85.0 & 4.9 & 40.2 & 39.4 & - & - \\
    {G-SchNet} & 95.7 & 68.1 & 85.5 & 80.3 & - & - \\

    \midrule[0.3pt]

    {EDM} &  98.7 & 82.0 & 91.9 & 90.7 & 81.3 & 92.6 \\

    \textbf{EDM+UCD} &  \textbf{98.8$\pm$0.0} & \textbf{87.2$\pm$0.2} & \textbf{94.3$\pm$0.1} & \textbf{92.9$\pm$0.2} & \textbf{83.5} & \textbf{95.7} \\

    
    {GeoLDM} &  98.9 & 89.4 & 93.8 & 92.7 & 84.4 & 99.3 \\

    \textbf{GeoLDM+UCD} & \underline{\textbf{99.1$\pm$0.0}} & \underline{\textbf{91.4$\pm$0.1}} & \textbf{95.7$\pm$0.2} & \underline{\textbf{93.1$\pm$0.2}} & \textbf{85.9} & \underline{\textbf{99.5}} \\

    {RADM} &  98.5 & 87.3 & 94.1 & 91.7 & 85.0 & 99.3 \\

    \textbf{RADM+UCD} &  \textbf{98.9$\pm$0.0} & \textbf{88.4$\pm$0.2} & \underline{\textbf{96.4$\pm$0.3}} & \textbf{91.8$\pm$0.2} & \underline{\textbf{89.2}} & \textbf{99.4} \\
    \bottomrule[1.0pt]
\end{tabular}}

\end{threeparttable}
}
\end{table*}

\subsection{Setups}
\label{sec:setups}
\paragraph{Datasets.} We evaluate the effectiveness of our method on two widely used standard benchmarks for 3D molecular generation: QM9~\cite{QM9} and GEOM-Drugs~\cite{GEOM-Drugs}, following standard experimental protocols adopted in prior work~\cite{EDM,GeoLDM,RADM}.
The QM9 dataset contains approximately 130K small organic molecules, each with up to 29 atoms.
We follow the standard split from~\citet{EDM}, using 100K molecules for training, 18K for validation, and 13K for testing.
GEOM-Drugs is substantially larger, consisting of around 420K molecules with an average of 44.4 atoms and up to 181 atoms per molecule.
Following~\citet{EDM}, we select the 30 lowest-energy conformations for each molecule.
Together, these two datasets cover a wide range of molecular sizes and complexities, providing a comprehensive evaluation setting.
\paragraph{Evaluation metrics.} Consistent with prior work~\cite{EDM,GeoLDM,RADM}, we assess generation quality using a suite of standard molecular metrics: \emph{atom stability}, \emph{molecule stability}, \emph{validity}, and \emph{validity$\times$uniqueness}~\cite{graphvae,ENF}. \textit{Atom Stability} measures the percentage of atoms whose number of bonds matches their chemical valence (e.g., H:1, C:4, O:2). \textit{Molecule Stability} reports the percentage of molecules in which all atoms are stable. \textit{Validity} denotes the fraction of molecules that satisfy valence constraints for all atoms, while \textit{Uniqueness} measures the proportion of distinct molecules among valid samples.
For GEOM-Drugs, following prior work, we omit molecule stability and uniqueness metrics, as they are consistently close to $0\%$ and $100\%$, respectively, across all evaluated methods.
\paragraph{Baselines.}
We consider several representative state-of-the-art diffusion models for 3D molecular generation as backbone architectures, including EDM~\citep{EDM}, GeoLDM~\citep{GeoLDM}, and RADM$_\text{DiT-B}$~\citep{RADM}. These models span a diverse range of design choices, including GNN-based and Transformer-based architectures, equivariant and non-equivariant formulations, and generation in either Euclidean or latent spaces. We apply our method as a plug-in module to each backbone and compare against their original implementations without uncertainty correction. In addition, we report results from non-diffusion-based baselines, including ENF~\citep{ENF} and G-SchNet~\citep{G-SchNet}, as provided in their original papers.
\paragraph{Implementation details.} To demonstrate the plug-in nature of our approach and ensure a fair comparison, we strictly use the officially released model settings for all backbone diffusion models. We do not modify any hyperparameters unrelated to Bayesian dropout inference or noise calibration, including noise schedules, encoder--decoder architectures, or dataset partitions.
\subsection{Unconditional Generation} 
\label{sec:mainresults}
Following standard evaluation protocols in 3D molecular generation~\cite{EDM,GeoLDM,RADM}, we generate 10,000 molecules for each model on both QM9 and GEOM-Drugs. All results reported in Table~\ref{tab:mainresults} are evaluated directly on the generated samples without any post-hoc refinement using computational chemistry tools (e.g., Open Babel). For QM9, we report the mean and standard deviation over three independent runs to account for randomness in sampling.
For GEOM-Drugs, performance variance across runs is negligible, and we therefore report single-run results, consistent with prior work.

\textbf{The consistent performance gains observed in Table~\ref{tab:mainresults} across all standard evaluation metrics highlight the urgency of addressing the long-overlooked uncertainty issue in 3D molecular generation.} This phenomenon is universal and cannot be resolved solely through architectural design choices. In particular, the issue persists across diverse modeling paradigms, including GNN-based equivariant EDM~\citep{EDM}, GeoLDM~\citep{GeoLDM}, and Transformer-based non-equivariant RADM~\citep{RADM}. Notably, even for GeoLDM and RADM, which operate in latent space, variance inflation remains consistently evident.
These results demonstrate that UCD effectively mitigates the detrimental impact of epistemic uncertainty in molecular diffusion. Importantly, the observed improvements are achieved without any modification to the underlying model architectures or training targets. As such, UCD provides a model-agnostic, inference-time calibration that complements existing architectural innovations rather than replacing them. Overall, augmenting existing diffusion models with UCD establishes new state-of-the-art performance for 3D molecular diffusion across multiple benchmarks. 

\subsection{Conditional Generation}
To further evaluate the ability of our method to generate molecules with desired properties, we consider conditional generation tasks on QM9 following prior work~\cite{EDM,GeoLDM,RADM}. We report the Mean Absolute Error (MAE) computed using pretrained property predictors, where a smaller deviation from the QM9 ground-truth distribution (lower bound) indicates better performance. Each target property requires training a separate conditional diffusion model. Due to computational constraints, we conduct conditional generation experiments only on the strongest baseline diffusion model, RADM~\cite{RADM}.

The results are summarized in Table~\ref{tab:cond}. \textbf{Across all six conditional tasks, RADM+UCD consistently outperforms the baselines, demonstrating improved accuracy in property-controlled molecular generation.} Combined with the unconditional results in Sec.~\ref{sec:mainresults}, these findings indicate that our approach not only generates more stable molecules, but also more faithfully satisfies target property constraints. Additional experimental details are provided in Appendix~\ref{app:conditional}.

\begin{table}[t]
\caption{Mean Absolute Error (MAE) for molecular property prediction using a pretrained predictor (lower is better). Baseline results are taken from the original papers. The best results are highlighted in \textbf{bold}.}
\label{tab:cond}
\centering
\resizebox{\columnwidth}{!}{
    \begin{tabular}{lcccccc} \toprule
    Property & $\alpha$ & $\Delta\varepsilon$ & $\varepsilon_{\text{HOMO}}$ & $\varepsilon_{\text{LUMO}}$ & $\mu$ & $C_v$ \\
    Unit & Bohr$^3$ & meV & meV & meV & D & $\frac{\text{cal}}{\text{mol}}$K \\ \midrule
    QM9 & 0.10 & 64 & 39 & 36 & 0.043 & 0.040 \\ \hline
    Random & 9.01 & 1470 & 645 & 1457 & 1.616 & 6.857 \\ 
    EDM & 2.76 & 655 & 356 & 584 & 1.111 & 1.101 \\
    GeoLDM & 2.37 & 587 & 340 & 522 & 1.108 & 1.025 \\ 
    RADM & 1.98& 458 & 290 & 383 & 0.814 & 
    0.869 \\ \hline
    RADM+UCD & \textbf{1.94} & \textbf{439} & \textbf{280} & \textbf{368} & \textbf{0.773} & \textbf{0.862}\\
    \bottomrule
    \end{tabular}
    }
\end{table}

\subsection{Efficiency and Ablation on Dropout Passes}
\label{sec:timeefficiency}

UCD introduces minimal additional complexity and only one new hyperparameter: the number of Monte Carlo dropout passes $M$ used for epistemic uncertainty estimation under the Bayesian approximation~\cite{dropout}. In this section, we jointly analyze the computational efficiency and generation quality as a function of $M$.

\paragraph{Training and sampling efficiency.}
For training efficiency, replacing standard linear layers with dropout-enabled layers \textbf{does not increase training cost}. Although dropout is active during training, it requires only a single forward pass per iteration and introduces no additional learnable parameters, resulting in negligible overhead. Moreover, UCD does not alter the diffusion sampling procedure or the total number of reverse steps. As a result, \textbf{the Number of Function Evaluations (NFEs), which is the standard measure of diffusion sampling efficiency, remains unchanged.}

\paragraph{Inference-time overhead.}
The primary computational cost of UCD arises during inference, where epistemic uncertainty is estimated via Monte Carlo dropout. Specifically, estimating uncertainty requires $M$ stochastic forward passes of the denoiser at each reverse step, increasing the \emph{per-step denoiser forward time} while keeping the number of reverse steps fixed. While the original Bayesian dropout formulation applies dropout at every layer, prior work has shown that inserting dropout only in selected layers (e.g., near the network output) is sufficient to obtain reliable uncertainty estimates with low computational cost~\cite{VariationalBayesianlastlayers,Bayesianlastlayernetworks}.

We report both the wall-clock time per single NFE and the corresponding generation quality on QM9 in Table~\ref{tab:efficiency_ablation_M}, measured using the widely adopted foundation model EDM~\cite{EDM} as the backbone. As the number of dropout passes $M$ increases, the per-step inference cost grows due to repeated forward evaluations of the denoiser. In practice, hardware-level parallelism and batch processing keep the overhead modest for small $M$, while beyond $M=10$ the runtime increases approximately linearly. Detailed settings are provided in Appendix~\ref{app:Efficiency}.

At the same time, Table~\ref{tab:efficiency_ablation_M} shows that molecular generation quality improves as $M$ increases from small values due to the improved accuracy of estimated epistemic uncertainty. However, the gains saturate once $M$ reaches $10$. Table~\ref{tab:nfe_time} further reports the corresponding wall-clock time per NFE across different diffusion backbones. While the absolute runtime varies substantially across models due to architectural differences, all backbones exhibit a consistent trend: inference cost increases mildly for small $M$. In particular, \textbf{large $M$ leads to diminishing quality gains but disproportionately higher runtime.} These results indicate that increasing $M$ improves the accuracy of epistemic uncertainty estimation. However, this improvement saturates once $M>10$, beyond which additional dropout passes provide diminishing benefits while incurring higher computational cost. Accordingly, we set $M=10$ as the default choice in our method.

\begin{table}[t]
\centering
\caption{Efficiency--quality trade-off under different numbers of dropout passes $M$ using EDM as the backbone on QM9.
Time is reported as wall-clock milliseconds per single NFE, while all molecular generation metrics are reported as percentages. $\uparrow$ and $\downarrow$ denote metrics where higher values and lower values indicate better performance, respectively.}
\label{tab:efficiency_ablation_M}
\setlength{\tabcolsep}{2pt}
\resizebox{\columnwidth}{!}{
\begin{tabular}{c | c | c c c c}
\toprule
$M$ & Time (ms) $\downarrow$& Atom Sta (\%) $\uparrow$& Mol Sta (\%) $\uparrow$& Valid (\%) $\uparrow$& Valid$\times$Unique (\%) $\uparrow$\\
\midrule
3   & 79.28 & 98.7 & 86.7 & 93.8 & 92.1 \\
5   & 79.66 & 98.7 & 87.0 & 94.0 & 92.4 \\
10  & 81.47 & 98.8 & 87.2 & 94.3 & 92.9 \\
20  & 88.17 & 98.8 & 87.3 & 94.6 & 93.1 \\
30  & 99.54 & 98.8 & 87.4 & 94.6 & 93.0 \\
\bottomrule
\end{tabular}
}
\end{table}

\vspace{8pt}
\begin{table}[H]
\caption{Wall-clock time (ms) per single NFE under different dropout configurations. }
\vspace{-4pt}
\label{tab:nfe_time}
\centering
\small
\setlength{\tabcolsep}{4pt}
\resizebox{\columnwidth}{!}{
\begin{tabular}{lccccc}
\toprule
Method &  $M{=}3$ & $M{=}5$ & $M{=}10$ & $M{=}20$ & $M{=}30$\\
\midrule
EDM+UCD  & 79.28 &  79.66  &  81.47  &  88.17 & 99.54 \\
GeoLDM+UCD  & 74.32 &  78.82  &  80.39  &  84.03 & 95.94\\
RADM+UCD  & 24.29 & 25.51 & 30.10 & 44.67 & 61.36\\

\bottomrule
\end{tabular}}
\end{table}

\balance
\subsection{Comparison with Alternative Uncertainty Estimators}

We further compare dropout-based uncertainty estimation with a
last-layer Laplace approximation, using EDM on QM9 as the backbone.
As shown in Table~\ref{tab:laplace}, both uncertainty estimators improve
over the original EDM sampler, confirming that accounting for epistemic
uncertainty is beneficial for molecular diffusion. However, dropout-based
estimation yields substantially stronger gains, especially on molecule
stability and validity. This result is consistent with our discussion in Sec.~\ref{similarworks}: last-layer Laplace approximation relies on a local linearization
of the model output and therefore captures only readout-level uncertainty,
whereas dropout provides a lightweight stochastic approximation that
better reflects uncertainty in the denoising network during reverse
sampling.

\begin{table}[H]
\centering
\caption{Comparison of uncertainty estimators on QM9 using EDM as the backbone.}
\label{tab:laplace}
\resizebox{\linewidth}{!}{
\begin{tabular}{lcccc}
\toprule
Model & Atom Sta (\%) $\uparrow$ & Mol Sta (\%) $\uparrow$ & Valid (\%) $\uparrow$ & Valid$\times$Unique (\%) $\uparrow$ \\
\midrule
EDM & 98.7 & 82.0 & 91.9 & 90.7 \\
EDM+Laplace & 98.7 & 82.9 & 93.0 & 91.6 \\
EDM+Dropout & 98.8 & 87.2 & 94.3 & 92.9 \\
\bottomrule
\end{tabular}
}
\end{table}

\balance
\section{Conclusion and Future Work}
In this work, we identify a long-overlooked harmful interaction between epistemic uncertainty from the learned denoiser and aleatoric uncertainty introduced by the diffusion process in 3D molecular generation, and formalize its effect as a variance inflation issue in reverse diffusion inference. Building on the theoretical analysis characterizing how this interaction degrades sampling quality, we propose \emph{Uncertainty-Calibrated Diffusion} (UCD), a model-agnostic, inference-time calibration method that directly mitigates this issue. Extensive experimental results demonstrate that augmenting existing diffusion models with UCD consistently improves molecular generation quality across multiple backbones and benchmarks.
Broadly speaking, investigating epistemic uncertainty, together with explainability~\cite{liu20253dgraphx,qu2025rise}, lays the foundation of trustworthy AI for 3D molecular modeling and generation.

This work focuses on diffusion-based \emph{3D molecular} generation, where the term \emph{molecule} refers to relatively small organic molecules such as those in QM9 and GEOM-Drugs. We do not evaluate UCD on macromolecular systems such as 3D materials~\cite{yan2022periodic} or proteins~\cite{wang2023learning}, which form a distinct generation task with different modeling choices and evaluation protocols. This scope is important because small-molecule configuration distributions are highly concentrated and effectively discontinuous with respect to chemical validity: small perturbations to atomic positions can break valency constraints or local chemical rules, while precise but minor changes can yield different valid molecules~\cite{controllable,chemicalspace,molecularspace,diffms}. Consequently, accurate recovery of fine-grained geometry during reverse diffusion is crucial, and uncertainty-induced variance inflation can have a significant effect on stability and validity.

Particularly, extending UCD to protein-level diffusion remains an important but non-trivial direction. Protein conformational spaces are typically smoother and more tolerant to local perturbations; for example, local fluctuations such as side-chain motions often do not alter the global fold, and geometric noise can sometimes be absorbed rather than producing a different valid structure. Protein generation also involves long-range interactions, hierarchical constraints, higher-dimensional configuration spaces, and different structural or functional metrics. A promising future direction is to adapt UCD to hierarchical protein diffusion pipelines by calibrating uncertainty at multiple structural scales, such as backbone and side-chain levels~\cite{wang2023learning}, and by incorporating long-range interaction structure into the uncertainty estimator. We further believe that relating learned uncertainty estimates to underlying chemical structure may yield deeper insight into molecular stability. Overall, our findings highlight that explicitly accounting for epistemic uncertainty at inference time is a critical ingredient for building robust and reliable generative models.








\clearpage



\bibliographystyle{ACM-Reference-Format}
\bibliography{icml2026/references, icml2026/airs}

\appendix
\section*{Appendix}


\section{Proofs}

Here we provide the proofs of Proposition~\ref{prop:effective_kernel}, Proposition~\ref{thm:path_kl} and Proposition~\ref{prop:endpoint_deviation}.

\subsection{Proposition 3.1}
\effectivekernel*
\label{proof3_1}
\begin{proof}

Condition on $\boldsymbol{z}_t$.
The quantity $f_t(\boldsymbol{z}_t)$ is then deterministic, and
the remaining randomness in Eq.~\eqref{eq:reverse_with_uncertainty} is the sum of two independent Gaussian vectors:
$-\kappa_t \boldsymbol{\delta}_t$ and $\sigma_t \boldsymbol{\varepsilon}$.

Since $\boldsymbol{\delta}_t \sim \mathcal{N}(\mathbf{0}, \tau_t^2 I_d)$, scaling by $-\kappa_t$ gives
\begin{equation}
-\kappa_t \boldsymbol{\delta}_t \sim \mathcal{N}(\mathbf{0}, \kappa_t^2 \tau_t^2 I_d).
\label{eq:app_scaled_delta}
\end{equation}
Similarly,
\begin{equation}
\sigma_t \boldsymbol{\varepsilon} \sim \mathcal{N}(\mathbf{0}, \sigma_t^2 I_d).
\label{eq:app_scaled_eps}
\end{equation}
Since $\boldsymbol{\delta}_t$ and $\boldsymbol{\varepsilon}$ are independent, where $\boldsymbol{\delta}_t$ captures epistemic uncertainty of the learned model itself and $\boldsymbol{\varepsilon}$ denotes the stochastic noise injected at each timestep, their linear combination remains Gaussian. In particular,
\begin{equation}
-\kappa_t \boldsymbol{\delta}_t + \sigma_t \boldsymbol{\varepsilon}
~\sim~
\mathcal{N}\!\left(\mathbf{0},\,(\kappa_t^2\tau_t^2+\sigma_t^2)I_d\right).
\label{eq:app_sum_gauss}
\end{equation}
Therefore, from Eq.~\eqref{eq:reverse_with_uncertainty} we obtain the conditional distribution
\begin{equation}
\boldsymbol{z}_{t-1}\mid \boldsymbol{z}_t
~\sim~
\mathcal{N}\!\left(
f_t(\boldsymbol{z}_t),\,
(\sigma_t^2+\kappa_t^2\tau_t^2)I_d
\right).
\label{eq:app_conditional_final}
\end{equation}
Defining $\eta_t^2:=\kappa_t^2\tau_t^2$ and $\tilde\sigma_t^2:=\sigma_t^2+\eta_t^2$ gives
\[
\tilde K_t(\boldsymbol{z}_{t-1}\mid \boldsymbol{z}_t)
=
\mathcal{N}\!\left(
f_t(\boldsymbol{z}_t),\,
\tilde\sigma_t^2 I_d
\right),
\]
which is exactly Eq.~\eqref{eq:unc_kernel}.

Finally, the baseline (no-uncertainty) kernel follows directly from Eq.~\eqref{eq:reverse_step_uncertainty}:
\begin{equation}
\boldsymbol{z}_{t-1}\mid \boldsymbol{z}_t
~\sim~
\mathcal{N}\!\left(
f_t(\boldsymbol{z}_t),\,
\sigma_t^2 I_d
\right),
\end{equation}
which matches Eq.~\eqref{eq:base_kernel}.

\end{proof}

\subsection{Proposition 3.2}

\pathkl*
\label{proof3_2}
\begin{proof}
Let the baseline and uncertainty-inflated reverse chains be Markov measures on the path
$\boldsymbol{z}_{0:T}:=(\boldsymbol{z}_T,\boldsymbol{z}_{T-1},\dots,\boldsymbol{z}_0)$ with the same initialization $p_T$:
\begin{align}
P^{\text{base}}(\boldsymbol{z}_{0:T})
&=
p_T(\boldsymbol{z}_T)\prod_{t=1}^T K_t(\boldsymbol{z}_{t-1}\mid \boldsymbol{z}_t),
\label{eq:app_path_base}
\\
P^{\text{unc}}(\boldsymbol{z}_{0:T})
&=
p_T(\boldsymbol{z}_T)\prod_{t=1}^T \tilde K_t(\boldsymbol{z}_{t-1}\mid \boldsymbol{z}_t),
\label{eq:app_path_unc}
\end{align}
where (from Proposition~\ref{prop:effective_kernel})
\begin{align}
K_t(\boldsymbol{z}_{t-1}\mid \boldsymbol{z}_t)
&=\mathcal N\!\bigl(\boldsymbol{z}_{t-1};\, f_t(\boldsymbol{z}_t),\,\sigma_t^2 I_d\bigr),
\label{eq:app_Kt_def}
\\
\tilde K_t(\boldsymbol{z}_{t-1}\mid \boldsymbol{z}_t)
&=\mathcal N\!\bigl(\boldsymbol{z}_{t-1};\, f_t(\boldsymbol{z}_t),\,(\sigma_t^2+\eta_t^2) I_d\bigr)\\
&=\mathcal N\!\bigl(\boldsymbol{z}_{t-1};\, f_t(\boldsymbol{z}_t),\,(1+r_t)\sigma_t^2 I_d\bigr).
\label{eq:app_tildeKt_def}
\end{align}
By definition,
\begin{equation}
\mathrm{KL}\!\left(P^{\text{unc}}\,\|\,P^{\text{base}}\right)
=
\mathbb E_{\boldsymbol{z}_{0:T}\sim P^{\text{unc}}}
\left[
\log\frac{P^{\text{unc}}(\boldsymbol{z}_{0:T})}{P^{\text{base}}(\boldsymbol{z}_{0:T})}
\right].
\label{eq:app_KL_def}
\end{equation}
Substituting the Markov factorizations Eq.~\eqref{eq:app_path_base}-\eqref{eq:app_path_unc} and canceling the common initial
density $p_T(\boldsymbol{z}_T)$ yields
\begin{align}
\log\frac{P^{\text{unc}}(\boldsymbol{z}_{0:T})}{P^{\text{base}}(\boldsymbol{z}_{0:T})}
&=
\sum_{t=1}^T
\log\frac{\tilde K_t(\boldsymbol{z}_{t-1}\mid \boldsymbol{z}_t)}{K_t(\boldsymbol{z}_{t-1}\mid \boldsymbol{z}_t)}.
\label{eq:app_log_ratio_sum}
\end{align}
Plugging Eq.~\eqref{eq:app_log_ratio_sum} into Eq.~\eqref{eq:app_KL_def} and exchanging expectation and summation gives the standard
chain-rule form for Markov measures:
\begin{equation}
\begin{aligned}
\mathrm{KL}\!\left(P^{\text{unc}}\,\|\,P^{\text{base}}\right)
&=
\sum_{t=1}^T
\mathbb E_{(\boldsymbol{z}_t,\boldsymbol{z}_{t-1})\sim P^{\text{unc}}}
\left[
\log\frac{\tilde K_t(\boldsymbol{z}_{t-1}\mid \boldsymbol{z}_t)}{K_t(\boldsymbol{z}_{t-1}\mid \boldsymbol{z}_t)}
\right]\\
&=
\sum_{t=1}^T
\mathbb E_{\boldsymbol{z}_t\sim \tilde p_t}
\left[
\mathrm{KL}\!\left(\tilde K_t(\cdot\mid \boldsymbol{z}_t)\,\|\,K_t(\cdot\mid \boldsymbol{z}_t)\right)
\right],
\label{eq:app_chain_rule}
\end{aligned}
\end{equation}
where $\tilde p_t$ is the marginal of $\boldsymbol{z}_t$ under $P^{\text{unc}}$.

It remains to compute the one-step conditional KL.
For any fixed $\boldsymbol{z}_t$, both kernels in Equation\eqref{eq:app_Kt_def}-\eqref{eq:app_tildeKt_def} are Gaussians with the
same mean $f_t(\boldsymbol{z}_t)$ and isotropic covariances $\sigma_t^2 I_d$ and $(1+r_t)\sigma_t^2 I_d$.
Therefore,
\begin{align}
&\mathrm{KL}\!\left(\tilde K_t(\cdot\mid \boldsymbol{z}_t)\,\|\,K_t(\cdot\mid \boldsymbol{z}_t)\right)
=
\mathrm{KL}\!\left(\mathcal N(\boldsymbol{0},(1+r_t)\sigma_t^2 I_d)\,\|\,\mathcal N(\boldsymbol{0},\sigma_t^2 I_d)\right)
\nonumber\\
&=
\frac{1}{2}
\left(
\mathrm{tr}\!\left((\sigma_t^2 I_d)^{-1}( (1+r_t)\sigma_t^2 I_d)\right)
-d
+\log\frac{\det(\sigma_t^2 I_d)}{\det((1+r_t)\sigma_t^2 I_d)}
\right)
\nonumber\\
&=
\frac{1}{2}
\left(
d(1+r_t)-d+\log\frac{(\sigma_t^2)^d}{((1+r_t)\sigma_t^2)^d}
\right)
\nonumber\\
&=
\frac{d}{2}\Bigl(r_t-\log(1+r_t)\Bigr).
\label{eq:app_step_kl}
\end{align}
Since Eq.~\eqref{eq:app_step_kl} does not depend on $\boldsymbol{z}_t$, the expectation in Eq.~\eqref{eq:app_chain_rule} drops and we obtain
\begin{equation}
\mathrm{KL}\!\left(P^{\text{unc}}\,\|\,P^{\text{base}}\right)
=
\sum_{t=1}^T \frac{d}{2}\Bigl(r_t-\log(1+r_t)\Bigr).
\end{equation}
\end{proof}

\subsection{Proposition 3.4}

\endpointdeviation*
\label{proof3_4}
\begin{proof}
    We begin by expanding the squared norm:
\begin{align}
\|{\Delta}_{t-1}\|^2
&=
\|J_t {\Delta}_t + \eta_t \boldsymbol{u}_t\|^2 \nonumber\\
&=
\|J_t {\Delta}_t\|^2
+ 2\eta_t \langle J_t {\Delta}_t, \boldsymbol{u}_t \rangle
+ \eta_t^2 \|\boldsymbol{u}_t\|^2.
\label{eq:app_expand}
\end{align}

At timestep $t$, the deviation ${\Delta}_t$ is fully determined by randomness
used at later reverse steps, while $\boldsymbol{u}_t$ is freshly sampled at the
current step and is independent of ${\Delta}_t$.
Therefore,
\[
\mathbb E\!\left[\langle J_t {\Delta}_t, \boldsymbol{u}_t \rangle\right] = 0,
\qquad
\mathbb E\!\left[\|\boldsymbol{u}_t\|^2\right] = d.
\]

Taking expectation in Eq.~\eqref{eq:app_expand} yields
\begin{equation}
\mathbb E\|{\Delta}_{t-1}\|^2
=
\mathbb E\|J_t {\Delta}_t\|^2 + d\eta_t^2.
\label{eq:app_moment_recursion}
\end{equation}

By Assumption~\ref{ass:sensitivity}, the singular values of $J_t$ satisfy
$m_t \le \lambda_{\min}(J_t)$ and $\lambda_{\max}(J_t) \le L_t$.
Hence, for any vector $\boldsymbol{v}$,
\begin{equation}
m_t \|\boldsymbol{v}\| \le \|J_t \boldsymbol{v}\| \le L_t \|\boldsymbol{v}\|.
\label{eq:app_sv_bound}
\end{equation}
Applying this inequality to $\boldsymbol{v}={\Delta}_t$ and taking expectations gives
\begin{equation}
m_t^2\,\mathbb E\|{\Delta}_t\|^2
\le
\mathbb E\|J_t {\Delta}_t\|^2
\le
L_t^2\,\mathbb E\|{\Delta}_t\|^2.
\label{eq:app_Jt_bound}
\end{equation}

Combining Eq.~\eqref{eq:app_moment_recursion} and Eq.~\eqref{eq:app_Jt_bound} yields the two-sided recursion
\begin{equation}
m_t^2\,\mathbb E\|{\Delta}_t\|^2 + d\eta_t^2
\le
\mathbb E\|{\Delta}_{t-1}\|^2
\le
L_t^2\,\mathbb E\|{\Delta}_t\|^2 + d\eta_t^2.
\label{eq:app_two_sided_recursion}
\end{equation}

Unrolling Eq.~\eqref{eq:app_two_sided_recursion} from $t=T$ down to $t=1$ using ${\Delta}_T=\mathbf{0}$ yields
\begin{align}
\mathbb E\|{\Delta}_0\|^2
&\ge
d\sum_{t=1}^T \eta_t^2 \prod_{s=1}^{t-1} m_s^2,
\label{eq:app_lower_unroll}
\\
\mathbb E\|{\Delta}_0\|^2
&\le
d\sum_{t=1}^T \eta_t^2 \prod_{s=1}^{t-1} L_s^2,
\label{eq:app_upper_unroll}
\end{align}
where the empty product for $t=1$ is defined as $1$.

\end{proof}


\section{Experimental Details}

\subsection{Conditional Generation}
\label{app:conditional}

\paragraph{Task.}
Following prior work on conditional diffusion for molecular generation~\cite{EDM}, 
the goal of conditional generation is to sample molecular structures with desired target properties.
Formally, given a property value $c$, we aim to generate atomic positions and features $(\boldsymbol{x},\boldsymbol{h}) \sim p(\boldsymbol{x},\boldsymbol{h}\mid c)$.

\paragraph{QM9 Conditional Properties.}
We consider conditional generation with respect to several molecular properties from the QM9 dataset.
Specifically, $\alpha$ denotes the molecular polarizability, which measures the tendency of a molecule to acquire an induced dipole moment under an external electric field.
$\varepsilon_{\mathrm{HOMO}}$ and $\varepsilon_{\mathrm{LUMO}}$ correspond to the energies of the highest occupied and lowest unoccupied molecular orbitals, respectively, and their difference $\Delta\varepsilon = \varepsilon_{\mathrm{LUMO}} - \varepsilon_{\mathrm{HOMO}}$ defines the HOMO-LUMO gap.
The dipole moment $\mu$ quantifies the separation of positive and negative charge within a molecule.
Finally, $C_v$ denotes the heat capacity measured at $298.15\,\mathrm{K}$.

\paragraph{Method.}
Conditioning is incorporated by augmenting the denoiser with the target property.
Specifically, the reverse diffusion model is parameterized as
${\boldsymbol{\varepsilon}}^c_\theta = \phi(\boldsymbol{z}_t,t,c)$,
where $c$ is concatenated to the node features.
The forward diffusion process remains unchanged.
Training follows the standard variational diffusion objective, with the conditional evidence lower bound decomposing as
\[
\log p(\boldsymbol{x},\boldsymbol{h}\mid c)
\;\ge\;
\mathcal{L}_{c,0} + \mathcal{L}_{c,\text{base}} + \sum_{t=1}^T \mathcal{L}_{c,t},
\]
where the conditional loss terms differ from the unconditional case only through the conditioned denoiser.

At inference time, molecular size $N$ and property value $c$ are first sampled from their empirical joint distribution estimated on the training set.
The conditional diffusion model then generates $(\boldsymbol{x},\boldsymbol{h})$ given $(c,N)$.

\paragraph{Experimental Setup.}
All conditional experiments follow the same setup as the unconditional case and prior EDM-based methods.
An EGNN backbone is used for both the conditional generator and a separately trained property predictor.
The predictor is trained on a half disjoint split of the QM9 dataset and is used only for evaluation.

\subsection{Wall-clock Efficiency}
\label{app:Efficiency}
All experiments are conducted on QM9 with a fixed batch size of 64, on a single NVIDIA RTX A6000 GPU with CUDA~12.4. Backbones and the task are introduced in Sec.~\ref{experimentsetting}.

For uncertainty-calibrated diffusion (UCD), the primary factor affecting wall-clock time is the number of Monte Carlo dropout passes $M$ used to estimate epistemic uncertainty.
Since UCD does not modify the diffusion schedule or the number of sampling steps, the total NFE count remains unchanged.
We therefore focus on per-step inference time to quantify the computational overhead introduced by uncertainty estimation.


\end{document}